%% file: paper.tex
\documentclass{article}

\PassOptionsToPackage{round}{natbib}
\usepackage[final]{neurips_2020}

\usepackage[utf8]{inputenc} 
\usepackage[T1]{fontenc}    
\usepackage{hyperref}       
\usepackage{url}            
\usepackage{booktabs}       
\usepackage{amsfonts}       
\usepackage{nicefrac}       
\usepackage{microtype}      
\usepackage{xspace}

\usepackage{graphicx}
\usepackage{enumitem}
\usepackage{subfig}
\usepackage{amsmath}
\usepackage{amssymb}
\usepackage{amsthm}
\usepackage{bm}
\usepackage{algorithm}
\usepackage{algpseudocode}

\usepackage{cleveref}

\title{Iterative Correction of Sensor Degradation and a Bayesian Multi-Sensor Data Fusion Method}

\author{
    Luka Kolar\thanks{All authors equally contributed to this work.} \quad Rok Šikonja\footnotemark[1] \quad Lenart Treven\footnotemark[1] \\
    ETH Zürich \\
    \texttt{\href{mailto:kolarl@student.ethz.ch}{kolarl@student.ethz.ch}} \quad
    \texttt{\href{mailto:rsikonja@student.ethz.ch}{rsikonja@student.ethz.ch}} \quad 
    \texttt{\href{mailto:trevenl@student.ethz.ch}{trevenl@student.ethz.ch}} \\
}

\date{\today}

\newtheorem{proposition}{Proposition}
\newtheorem{theorem}{Theorem}
\newtheorem{remark}{Remark}
\newtheorem{corollary}{Corollary}

\input{notation}

\begin{document}

\maketitle

\begin{abstract}
\input{abstract}
\end{abstract}

\section{Introduction}
\label{section:Introduction}
\input{introduction}

\section{Degradation Modeling and Correction}
\label{section:DegradationModeling}
\input{degradation}

\section{Convergence Theorems}
\label{section:ConvergenceThreorems}
\input{convergence_theorems}

\section{Data Fusion}
\label{section:DataFusion}
\input{data_fusion}

\section{Results}
\label{section:Results}
\input{results}

\section{Conclusion}
\label{section:Conclusion}
\input{conclusion}

\bibliographystyle{apalike}
\bibliography{references}
\newpage

\appendix
\section{Appendix}
\label{section:Appendix}
\input{appendix}

\end{document}

%% file: notation.tex
\let\vec\mathbf

\newcommand{\R}{\mathbb{R}}

\newcommand{\T}{\mathsf{T}}
\DeclareMathOperator{\diag}{\mathrm{diag}}
\DeclareMathOperator{\tr}{\mathrm{tr}}
\newcommand{\norm}[2]{\left\lVert#1\right\rVert_{#2}}

\newcommand{\siga}{a}
\newcommand{\sigb}{b}

\newcommand{\sigs}{s}
\newcommand{\signa}{{n_a}}
\newcommand{\signb}{{n_b}}

\newcommand{\signm}{{n_m}}
\newcommand{\expa}{e_a}
\newcommand{\expb}{e_b}
\newcommand{\noisea}{\varepsilon_a}
\newcommand{\noiseb}{\varepsilon_b}

\newcommand{\timeseries}[3]{\left ( #1[#2]\right)_{#2 \in [#3]}}

\newcommand{\params}{\bm{\theta}}
\newcommand{\sparams}{\theta}

\newcommand{\algcorrectone}{\textsc{CorrectOne}\xspace}
\newcommand{\algcorrectboth}{\textsc{CorrectBoth}\xspace}

\newcommand{\fitcurve}{\textsc{FitCurveTo}\xspace}

\newcommand{\expmodel}{\textsc{Exp}\xspace}
\newcommand{\explinmodel}{\textsc{ExpLin}\xspace}

\newcommand{\isotonicmodel}{\textsc{Isotonic}\xspace}
\newcommand{\smoothmonmodel}{\textsc{SmoothMonotonic}\xspace}

\newcommand{\correctionmodel}{\textsc{Correction}\xspace}

\newcommand{\svgp}{\textsc{SparseVariationalGP}\xspace}

%% file: abstract.tex
We present a novel method for inferring ground-truth signal from multiple degraded signals, affected by different amounts of sensor ``exposure''.
The algorithm learns a multiplicative degradation effect by performing iterative corrections of two signals solely from the ratio between them.
The degradation function $d$ should be continuous, satisfy monotonicity, and $d(0) = 1$.
We use smoothed monotonic regression method, where we easily incorporate the aforementioned criteria to the fitting part.
We include theoretical analysis and prove convergence to the ground-truth signal for the noiseless measurement model. 
Lastly, we present an approach to fuse the noisy corrected signals using Gaussian processes.
We use sparse Gaussian processes that can be utilized for a large number of measurements together with a specialized kernel that enables the estimation of noise values of all sensors.
The data fusion framework naturally handles data gaps and provides a simple and powerful method for observing the signal trends on multiple timescales (long-term and short-term signal properties).
The viability of correction method is evaluated on a synthetic dataset with known ground-truth signal.

%% file: introduction.tex
The idea of the paper is inspired by the problem of radiometer degradation correction arising in the measurement process of total solar irradiance (TSI), such as in Variability of Solar Irradiance and Gravity Oscillations (VIRGO) experiment on Solar and Heliospheric Observatory (SOHO) spacecraft.

The aim is to reconstruct the ground-truth signal from two time series of measurements, which are produced by two identical sensors that have different sampling rates and are thus subjected to different amounts of exposure to radiation. 
Let's call the sensor with higher and lower sampling rate \emph{the main} and \emph{the back-up} sensor, respectively, following the convention of \citet{anklin1998}.

The exposure causes the sensors to degrade, i.e. their sensitivity decreases.
Since the two sensors are identically built, the degradation affects them at the same rate,
but both attain the same amount of degradation at different times due to different amount of exposure.
Furthermore, measurements are affected by a random additive noise, which we assumed to be Gaussian.

The goal of this paper is to devise a general and unified method to conduct \emph{a posteriori} degradation correction, and to combine information from two, or even multiple, noisy, corrected, signals into a reliable estimate of the ground-truth.

\subsection{Related Work}
\label{sec:related_work}

Several methods have been proposed for degradation correction of radiometers~\citep{frohlich2014, frohlich2003, anklin1998}, however these only address radiometers and not sensors in general and incorporate many modeling assumptions.
Furthermore, a wide variety of data-fusion techniques are available and well understood, and can be found in an extensive survey by \citet{castanedo2013}. 
Our data-fusion method explores the capabilities of Gaussian processes, proposed  by \citet{rasmussen2005}.
The state-of-the-art method applied for data fusion of TSI time series is proposed by \citet{dudokdewit2017}, and utilizes a data-driven multi-scale maximum likelihood method. 
Many domain specific methods often suffer from preconceived modeling bias and thus direct our work to develop a more general -- data-driven -- method.

The rest of this paper is organized as follows.
\Cref{section:DegradationModeling} describes the methods of degradation modeling and proposes two correction methods.
\Cref{section:ConvergenceThreorems} gives a theoretical justification for the proposed correction methods and establishes their convergence.
In \cref{section:DataFusion} methods for Bayesian data fusion are presented and described.
Finally, in \cref{section:Results} proposed methods are evaluated on a synthetic dataset.

\subsection{Notation}
\label{subsection:TimeSeriesNotation}

Throughout the paper, continuous and discrete time signals are differentiated: a parentheses notation denotes a continuous signal, whereas a square bracket notation denotes a discrete time series.
Continuous-time notation is used in development of the theory, whereas discrete-time in the description of the algorithms.

The signal from the main sensor is denoted as $\siga$ and from the back-up sensor as $\sigb$.
We assume that the main sensor is at any point in time except at the very beginning subjected to more exposure than the back-up sensor.
The indexed set $T_\siga = \{t^\siga_1, \ldots, t^\siga_\signa\}$ contains all sampling times of the sensor $\siga$, where $\signa$ denotes the number of measurements, and $\timeseries{\siga}{k}{\signa}$ the corresponding time series of measurements.
Similarly, we have for sensor $\sigb$.
We denote the ground-truth signal as $\sigs$.

%% file: degradation.tex
Degradation is modeled as a multiplicative effect on the sensitivity of both sensors and is described by the degradation function $d(\cdot)$.
It is a function of what we call the exposure or exposure ``time'' $e$. It takes the value of 1 at time 0, i.e. $d(0) = 1$, and is assumed to be a monotonically decreasing, continuous function.
The first assumption implies that instruments are not degraded at time $t = 0$, and the second corresponds to the fact that sensor performance does not improve with exposure to radiation.
As it turns out, the first assumption is of paramount importance in the modeling of degradation, not only being intuitively true, but providing a necessity for degradation correction as we show in \cref{section:ConvergenceThreorems}.
The second assumption improves interpretability and robustness of the proposed method.
The information on exposure of each sensor is supplied by the expert conducting the correction, here we use an estimate of cumulative sum of measured values.
We denote the exposure functions for the two sensors by $e_a(\cdot)$ and $e_b(\cdot)$.

Noiseless measurement model is defined in \cref{eq:degradation_model} and is used for the development of theory of degradation correction.
\begin{align}
\label{eq:degradation_model}
\begin{aligned}
\siga(t) &= \sigs(t) \cdot d(\expa(t)),\\
\sigb(t) &= \sigs(t) \cdot d(\expb(t)). \\
\end{aligned}
\end{align}

In order to get a realistic model, which accounts for the noisy nature of measurements, we introduce normally distributed additive noise with zero mean and a constant sensor-dependent variance, $\sigma_\siga^2$ and $\sigma_\sigb^2$.
We denote by $\noisea$ and $\noiseb$ white noise signals (zero mean and independent in time) that are independent of each other.
This yields noisy measurement model defined in \cref{eq:degradation_model_noise}.
\begin{align}
\label{eq:degradation_model_noise}
\begin{aligned}
\siga(t) &= \sigs(t) \cdot d(\expa(t)) + \noisea(t), && \noisea(t) \sim \mathcal{N}(0,\,\sigma_\siga^{2}), \\
\sigb(t) &= \sigs(t) \cdot d(\expb(t)) + \noiseb(t), && \noiseb(t) \sim \mathcal{N}(0,\,\sigma_\sigb^{2}).
\end{aligned}
\end{align}
 Note that we have $\expa(0) = \expb(0) = 0$.

\subsection{Degradation Modeling}
\label{sec:degradation_modeling}

The degradation function is determined solely from the ratio of signals $r(t) = \frac{a(t)}{b(t)}$ or more precisely from its discrete time counterpart $(r[k] = \frac{a[k]}{b[k]})_{k \in \signm}$, where $t_k \in T_m = \{t^m_1, \ldots, t^m_{\signm} \} = T_\siga \cap T_\sigb$.
Thus, the ratio is computed only at discrete times at which both sensors took a measurement.
However, our method could straightforwardly be extended to the case when there are no simultaneous measurements by using interpolation methods.

Next, we propose different modeling methods for learning the degradation function.
This can be formulated as a univariate regression problem, where we assume that there exists a function $f(\cdot, \params)$ parametrized by $\params$, which describes the relation between a predictor variable $x$ and a response variable $y$, taking the form of 
\begin{align}
\label{eq:regression_model}
\begin{aligned}
y = f(x, \params) + \varepsilon, \\
\end{aligned}
\end{align}
where $\varepsilon$ represents random noise with zero mean.

\subsubsection{Exponential families}

Many physical phenomena can be described by the family of exponential functions, because they arise naturally as the solutions to the differential equations.

Firstly, we propose a simple exponential function with parameters $\params = [\sparams_1, \sparams_2]^\T$, where $\sparams_1 > 0$ is the exponential decay and $\sparams_2$ is a scaling parameter, taking the form 
\begin{align}
\label{eq:exp_model}
\tag{\expmodel}
\begin{aligned}
f(x, \params) = 1 - e^{\sparams_1 \cdot \sparams_2} + e^{- \sparams_1 \cdot (x - \sparams_2)}.
\end{aligned}
\end{align}

Secondly, we propose an extension to \ref{eq:exp_model} by introducing an additional linear dependency term, taking the form
\begin{align}
\label{eq:exp_lin_model}
\tag{\explinmodel}
\begin{aligned}
f(x, \params) = 1 - e^{\sparams_1 \cdot \sparams_2} + e^{- \sparams_1 \cdot (x - \sparams_2)} + \sparams_3 \cdot x.
\end{aligned}
\end{align}
The latter model is adopted from the work of \citet{anklin1998}. Note that both models satisfy the condition $f(0, \params) = 1$  regardless of the values of $\params$.

\subsubsection{Monotonic Regression and Smoothed Monotonic Regression}

One of the main modeling challenges is the robustness of the proposed correction method.
By assumption, degradation function has to be a monotonically decreasing function and thus monotonic regression is a natural choice for its modeling.

Monotonic or isotonic regression requires predictor values to be in a strictly increasing order $x_1 < x_2 < \ldots < x_n$.
Its solution is a stepwise interpolating function determined by $n$ points $\params = [\sparams_1, \ldots, \sparams_n]^\T$ in monotonically decreasing order, which are obtained by solving the following optimization problem with $f(x_i, \params) = \sparams_i, \;\forall \, i \in [n]$, \citep{sysoev2019}
\begin{align}
\label{eq:monotonic_regression_problem}
\tag{\isotonicmodel}
\begin{aligned}
&\min_{\params} \sum_{i=1}^n \, \left (f(x_i, \params) - y_i \right )^2 = \min_{\params} \sum_{i=1}^n \, \left (\sparams_i - y_i \right )^2 \\
&\text{s. t. } \sparams_i \geq \sparams_{i+1} \text{ for $i \in \{1, \ldots, n-1\}$.}
\end{aligned}
\end{align}

This quadratic optimization problem provides a simple and powerful framework for enforcing additional constraints, such as $f(0, \params) = 1$ or convexity/concavity, which can be added to~\ref{eq:monotonic_regression_problem} as
\begin{align}
\label{eq:monotonic_regression_constraints}
\begin{aligned}
\sparams_1 = 1 && \text{to ensure} && f(0, \params) = 1, \\
\sparams_{i+1} - 2 \cdot \sparams_i + \sparams_{i-1} \geq 0 \text{ for $i \in \{2, \ldots, n-1\}$} && \text{to ensure} && \text{Convexity, } f''(x_i) \geq 0. \\ 
\end{aligned}
\end{align}

The practical issue with monotonic regression is that it resembles a discontinuous step function while we expect the function to be continuous and smooth.
In this regard, \citet{sysoev2019} propose a modification of \ref{eq:monotonic_regression_problem} by penalizing the difference between adjacent fitted response values, $\sparams_i$ and $\sparams_{i+1}$, by using an L2 regularization term.
This yields a smoothed monotonic regression problem formulated as
\begin{align}
\label{eq:smooth_monotonic_regression_problem}
\tag{\smoothmonmodel}
\begin{aligned}
& \min_{\params} \sum_{i=1}^n \, \left (\sparams_i - y_i \right )^2 + \sum_{i = 1}^{n-1} \lambda_i \cdot (\sparams_{i+1} - \sparams_i)^2 \\
&\text{s. t. } \sparams_i \geq \sparams_{i+1} \text{ for $i \in \{1, \ldots, n-1\}$},
\end{aligned}
\end{align}
where $\lambda_i$ for $i \in [n-1]$ are selected regularization parameters.
Note that similarly to~\ref{eq:monotonic_regression_problem} we can introduce additional constraints, such as~\eqref{eq:monotonic_regression_constraints}.
We choose $\lambda_1 = \ldots = \lambda_{n-1} = 1$.
The difference in naming models, isotonic vs. monotonic, is introduced only for the purpose of clarity. 

The stability of quadratic optimization problem~\ref{eq:smooth_monotonic_regression_problem} is guaranteed by first running~\ref{eq:monotonic_regression_problem}, then uniformly sampling the obtained function $f(\cdot, \params)$ at $\mathcal{X} = \{x^r_1, \ldots,x^r_m \}$, where $n \gg m$, and lastly running~\ref{eq:smooth_monotonic_regression_problem} with $\mathcal{D}' = \{(x^r_i, y^r_i)\}_{i=1}^m$.
Finally $f$ is obtained as a linear interpolation between each $x^r_i$ and $x^r_{i+1}$.

\subsection{Degradation Iterative Correction}

The process described in this section yields degradation-corrected versions of signals, denoted by $\siga_c$ and $\sigb_c$.
Here the noiseless measurement model, defined by~\cref{eq:degradation_model}, is used.
We discuss the rationale for not using~\cref{eq:degradation_model_noise} later.
Note that neither the ground-truth signal $\sigs$ nor the degradation function $d(\cdot)$ is known.
We present two methods with which $\sigs$ and $d(\cdot)$ can be approximately extracted just from time series $\timeseries{\siga}{k}{\signm}$ and $\timeseries{\sigb}{k}{\signm}$.
Method given by \cref{eq:correct_one} is described here, whereas method described with \cref{eq:correct_both} is analyzed in \cref{sec:ap_degradation_correction} for the sake of brevity.
In order to unambiguously show the computational steps on the given time series, we present these methods along with corresponding algorithms \ref{alg:correct_one} and \ref{alg:correct_both}.

Method $\algcorrectone$ performs iterative corrections of both signals, formulated as follows:
\begin{align}
\label{eq:correct_one}
\tag{\algcorrectone}
\begin{aligned}
r_n(\expa(t)) &= \frac{\siga_0(t)}{\sigb_n(t)} &
\siga_{n+1}(t) &= \frac{\siga_0(t)}{r_n(\expa(t))} &
\sigb_{n+1}(t) &= \frac{\sigb_0(t)}{r_n(\expb(t))} && \text{for $n = 0, 1, \ldots$}
\end{aligned}
\end{align}
where $\siga_0(t) = \siga(t) = \sigs(t) \cdot d(\expa(t))$ and $\sigb_0(t) = \sigb(t) = \sigs(t) \cdot d(\expb(t))$ holds.

Both algorithms return the best estimates of $\sigs$ for each instrument for $t_k \in T_m$ and $d(\cdot)$.
In section~\ref{section:ConvergenceThreorems} we show that indeed, $d_c(\cdot) \xrightarrow{} d(\cdot)$, $\siga_c(\cdot) \xrightarrow{} \sigs(\cdot)$ and $\sigb_c(\cdot) \xrightarrow{} \sigs(\cdot)$ holds.
After obtaining $d_c(\cdot)$ both time series $\timeseries{\siga}{k}{\signa}$, $\timeseries{\sigb}{k}{\signb}$ can be corrected as
\begin{align}
\label{eq:degradation_correction}
\tag{\correctionmodel}
\begin{aligned}
\siga_c(t) &= \frac{\siga(t)}{d_c(\expa(t))} \text{ and } \siga_c[k] = \frac{\siga[k]}{d_c(\expa[k])} && \text{for $k \in [\signa]$}, \\
\sigb_c(t) &= \frac{\sigb(t)}{d_c(\expb(t))} \text{ and } \sigb_c[k] = \frac{\sigb[k]}{d_c(\expb[k])} && \text{for $k \in [\signb]$}.
\end{aligned}
\end{align}

In $\algcorrectone$ at each iteration the ratio is computed between initial signal $\siga$ and corrected $b$, therefore the ratio is an approximation of the degradation function, and, in order to satisfy the modeling assumptions, robustness is enforced by fitting a function from the family of functions proposed in \autoref{sec:degradation_modeling}.
Empirically, the \smoothmonmodel has proven to be the most robust and realistic among all proposed.

\begin{algorithm}[ht]
\caption{$\algcorrectone(\timeseries{\siga}{k}{\signm}, \timeseries{\sigb}{k}{\signm}, \timeseries{\expa}{k}{\signm}, \timeseries{\expb}{k}{\signm})$} \label{alg:correct_one}
\begin{algorithmic}[1]
\State $a_c \gets a; \; b_c \gets b$ \Comment{Initial estimate of corrected signals}

\While{$\text{not converged}$} \Comment{E.g. $\norm{\siga_{i+1} - \siga_{i}}{2} /\norm{\siga_{i}}{2} + \norm{\sigb_{i+1} - \sigb_{i}}{2} / \norm{\sigb_{i}}{2} > \varepsilon$}
    \State $r \gets \frac{\siga}{\sigb_c}$ \Comment{Divide signals $\siga$ and $\sigb_c$ pointwise, i.e. $r[k] = \frac{\siga[k]}{\sigb_c[k]}$ $\forall \, k \in [\signm]$}
    \State $f(\cdot) \gets \fitcurve(\timeseries{\expa}{k}{\signm}, \timeseries{r}{k}{\signm})$  \Comment{Learn mapping $f \,:\, \expa \mapsto f(\expa)$}
    \State $\siga_c \gets \frac{\siga}{f(\expa)}$ \Comment{Correction update of signal $\siga$}
    \State $\sigb_c \gets \frac{\sigb}{f(\expb)}$ \Comment{Correction update of signal $\sigb$}
\EndWhile

\State $d_c(\cdot) \gets f(\cdot)$ \Comment{Final estimate of degradation function $d(\cdot)$}
\State \textbf{return} $\timeseries{\siga_c}{k}{\signm}, \timeseries{\sigb_c}{k}{\signm}, d_c(\cdot)$ \Comment{Return corrected signals and degradation function}
\end{algorithmic}
\end{algorithm}

%% file: convergence_theorems.tex
This section provides some theoretical guarantees for the convergence of our methods for the noiseless measurement model~\eqref{eq:degradation_model}.
We show this first for simple exposure functions $\expa(t) = t$ and $\expb(t) = \frac{t}{2}$, and then generalize to arbitrary exposure functions. Here we present the result only for \algcorrectone. The results and analysis for \algcorrectboth follow similar lines and hence are deferred to the appendix, \autoref{sec:ap_convergence_theorems}.

\begin{proposition}[$\algcorrectone$ for $\expa(t) = t$ and $\expb(t) = \frac{t}{2}$]
    \label{thm:correction_both}
    Let $a_0(t) = s(t) \cdot d(t)$ and $b_0(t) = s(t)  \cdot d(\frac{t}{2})$ for $t \geq 0$, where $s(t)>0$ is the ground-truth signal and $d \,:\, \R_{\geq 0} \to [0,1]$ is a continuous degradation function with $d(0) = 1$.
    If we run algorithm for $n=0,1,\ldots$ ~:
    \begin{align*}
        r_n(t) = \frac{a_0(t)}{b_n(t)}, ~~
        a_{n+1}(t) = \frac{a_0(t)}{r_n(t)}, ~~
        b_{n+1}(t) = \frac{b_0(t)}{r_n(\frac{t}{2})},
    \end{align*}
    then we have $\forall \, t \geq 0: \lim_{n \to \infty} a_n(t) = \lim_{n \to \infty}b_n(t) = s(t)$ and $\lim_{n \to \infty} r_n(t) = d(t)$.
\end{proposition}
\begin{proof}
Let us fix an arbitrary $t > 0$.
Since we have $r_0(t) = \frac{d(t)}{d(\frac{t}{2})}$ by induction follows:
\begin{align*}
r_n(t) &= \frac{a_0(t)}{b_n(t)} 
        = \frac{a_0(t)}{\frac{b_0(t)}{r_{n-1}(\frac{t}{2})}} 
        = \frac{a_0(t)}{b_0(t)} \cdot r_{n-1}\left(\frac{t}{2}\right)
        = \frac{a_0(t)}{b_0(t)} \cdot \frac{a_0(\frac{t}{2})}{b_0(\frac{t}{2})} \cdot r_{n-2}\left(\frac{t}{2^2}\right) \\
        &= \ldots
        = \prod_{i=0}^{n-1} \frac{a_0(\frac{t}{2^i})}{b_0(\frac{t}{2^i})} \cdot r_{0}\left(\frac{t}{2^n}\right)
        = \prod_{i=0}^{n} \frac{d(\frac{t}{2^i})}{d(\frac{t}{2^{i+1}})}
        = \frac{d(t)}{d(\frac{t}{2^{n+1}})}.
\end{align*}
Calculating $b_n(t)$ yields
\begin{align*}
b_n(t)  = \frac{b_0(t)}{r_{n-1}(\frac{t}{2})} 
        = \frac{s(t) \cdot d(\frac{t}{2})}{\frac{d(\frac{t}{2})}{d(\frac{t}{2^{n+1}})}}
        = s(t) \cdot d\left(\frac{t}{2^{n+1}}\right)
\end{align*}
and similarly for $a_n(t)$, we have
\begin{align*}
    a_n(t) = s(t) \cdot d\left(\frac{t}{2^{n}}\right).
\end{align*}
By taking the limit $n \xrightarrow{} \infty$, we get
\begin{align*}
    \lim_{n \to \infty}b_n(t) = \lim_{n \to \infty}s(t) \cdot d\left(\frac{t}{2^{n+1}}\right) = s(t) \cdot d\left(\frac{t}{\lim_{n \to \infty}2^{n+1}}\right) = s(t) \cdot d(0) = s(t),
\end{align*}
from where it immediately follows $\lim_{n \to \infty} a_n(t) = s(t)$.
Moreover, the ratio $r_n(t)$ converges to $d(t)$, which follows from continuity of $d$ as
\begin{align*}
    \lim_{n \to \infty} r_n(t) = \lim_{n \to \infty} \frac{d(t)}{d(\frac{t}{2^{n+1}})}
    =  \frac{d(t)}{\lim_{n \to \infty} d(\frac{t}{2^{n+1}})}
    = \frac{d(t)}{d(\frac{t}{\lim_{n \to \infty} 2^{n+1}})}
    = \frac{d(t)}{d(0)} = d(t).
\end{align*}
\end{proof}

Next, we will generalize the obtained result to an arbitrary exposure function by first stating a straightforward corollary, the proof of which immediately follows from the proposition~\ref{thm:correction_both}.

\begin{corollary}[$\algcorrectone$ for $\expa(t) = t$ and $\expb(t) = \frac{t}{k}$, $k > 1$]
    Let $a_0(t) = s(t) \cdot d(t)$ and $b_0(t) = s(t)  \cdot d(\frac{t}{k})$ for $t \geq 0$, where $s(t)>0$ is the ground-truth signal, $d \,:\, \R_{\geq 0} \to [0,1]$ is a continuous degradation function with $d(0) = 1$  and $k > 1$ is an arbitrary sampling rate parameter.
    If we run algorithm for $n=0,1,\ldots$ ~:
    \begin{align*}
        r_n(t) = \frac{a_0(t)}{b_n(t)}, ~~
        a_{n+1}(t) = \frac{a_0(t)}{r_n(t)}, ~~
        b_{n+1}(t) = \frac{b_0(t)}{r_n(\frac{t}{k})},
    \end{align*}
    then we have $\forall \, t \geq 0: \lim_{n \to \infty} a_n(t) = \lim_{n \to \infty}b_n(t) = s(t)$ and $\lim_{n \to \infty} r_n(t) = d(t)$.
\end{corollary}

\begin{proposition}[$\algcorrectone$ for $\expa(t) = t$ and $\expb(t) = e(t)$, $e(t) < t$]
    \label{thm:correction_both_general}
    Let $a_0(t) = s(t) \cdot d(t)$ and $b_0(t) = s(t)  \cdot d(e(t))$ for $t \geq 0$, where $s(t)>0$ is the ground-truth signal, $d \,:\, \R_{\geq 0} \to [0,1]$ is a continuous degradation function with $d(0) = 1$  and $e(t)$ is the exposure function of signal $b$, for which we have $e(0) = 0$ and $e(t) < t$ for all $t > 0$.
    If we run algorithm for $n=0,1,\ldots$ ~:
    \begin{align*}
        r_n(t) = \frac{a_0(t)}{b_n(t)}, ~~
        a_{n+1}(t) = \frac{a_0(t)}{r_n(t)}, ~~
        b_{n+1}(t) = \frac{b_0(t)}{r_n(e(t))},
    \end{align*}
    then we have $\forall \, t \geq 0: \lim_{n \to \infty} a_n(t) = \lim_{n \to \infty}b_n(t) = s(t)$ and $\lim_{n \to \infty} r_n(t) = d(t)$.
\end{proposition}
\begin{proof}
Let us fix an arbitrary $t > 0$ and define $e^n(t) = \underbrace{(e \circ e\circ \cdots \circ e)}_{n \text{ times}}(t)$. Since $r_0(t) = \frac{d(t)}{d(e(t))}$ we have by induction:
\begin{align*}
r_n(t) &= \frac{a_0(t)}{b_n(t)} 
        = \frac{a_0(t)}{\frac{b_0(t)}{r_{n-1}(e(t))}} 
        = \frac{a_0(t)}{b_0(t)} \cdot r_{n-1}(e(t))
        = \frac{a_0(t)}{b_0(t)} \cdot \frac{a_0(e(t))}{b_0(e(t))} \cdot r_{n-2}(e(e(t))) \\
        &= \ldots
        = \prod_{i=0}^{n-1} \frac{a_0(e^i(t))}{b_0(e^i(t))} \cdot r_{0}(e^n(t))
        = \prod_{i=0}^{n} \frac{d(e^i(t))}{d(e^{i+1}(t))}
        = \frac{d(t)}{d(e^{n+1}(t))}.
\end{align*}
Calculating $b_n(t)$ yields
\begin{align*}
b_n(t)  = \frac{b_0(t)}{r_{n-1}(e(t))} 
        = \frac{s(t) \cdot d(e(t))}{\frac{d(e(t))}{d(e^{n+1}(t))}}
        = s(t) \cdot d(e^{n+1}(t))
\end{align*}
and similarly for $a_n(t)$, we have
\begin{align*}
    a_n(t) = s(t) \cdot d(e^{n}(t)).
\end{align*}
By the same reasoning as in proposition~\ref{secondProposition}, we get $\lim_{n \to \infty} d(e^{n+1}(t)) = 0$ and thus by taking the limit $n \xrightarrow{} \infty$ on $b_n(t)$ we get
\begin{align*}
    \lim_{n \to \infty}b_n(t) = \lim_{n \to \infty}s(t) \cdot d\left(e^{n+1}(t)\right) = s(t) \cdot d\left(0\right) = s(t),
\end{align*}
from where it immediately follows $\lim_{n \to \infty} a_n(t) = s(t)$.
Moreover, the ratio $r_n(t)$ converges to $d(t)$, which follows from continuity of $d$ as
\begin{align*}
    \lim_{n \to \infty} r_n(t) = \lim_{n \to \infty} \frac{d(t)}{d(e^{n+1}(t))}
    =  \frac{d(t)}{\lim_{n \to \infty} d(e^{n+1}(t))}
    = \frac{d(t)}{d(0)} = d(t).
\end{align*}
\end{proof}

\begin{theorem}
    \label{thm:correctone}
    Let $a_0(t) = s(t) \cdot d(e_a(t))$ and $b_0(t) = s(t)  \cdot d(e_b(t))$ for $t \geq 0$, where $s(t)>0$ is the ground-truth signal, $d \,:\, \R_{\geq 0} \to [0,1]$ is a continuous degradation function with $d(0) = 1$ and $e_a(t), e_b(t):[0, \infty)\to[0, \infty)$ are the continuous exposure function of signal $a$ and $b$ respectively.
    Let us further assume $e_a(0) = e_b(0)= 0$, $e_b(t) < e_a(t)$ for all $t > 0$ and that there exist function $e_a^{-1}:[0, \infty)\to[0, \infty)$. If we run algorithm for $n=0,1,\ldots$ ~:
     \begin{align*}
        r_n(t) = \frac{a_0(t)}{b_n(t)}, ~~
        a_{n+1}(t) = \frac{a_0(t)}{r_n(t)}, ~~
        b_{n+1}(t) = \frac{b_0(t)}{r_n((e_a^{-1}\circ e_b)(t))},
    \end{align*}
    then we have $\forall \, t \geq 0: \lim_{n \to \infty} a_n(t) = \lim_{n \to \infty}b_n(t) = s(t)$.
\end{theorem}

\begin{proof}
    Let $h(t) = d(e_a(t))$. Then we have that $d(e_b(t)) = h(e_a^{-1}\circ e_b)(t)$. If we denote $e = e_a^{-1}\circ e_b$, then the proposed algorithm transforms to:
    \begin{align*}
        r_n(t) = \frac{a_0(t)}{b_n(t)}, ~~
        a_{n+1}(t) = \frac{a_0(t)}{r_n(t)}, ~~
        b_{n+1}(t) = \frac{b_0(t)}{r_n(e(t))},
    \end{align*}
    with the initial setting: $a_0(t) = s(t) \cdot h(t)$ and $b_0(t) = s(t)  \cdot h(e(t))$.
    Since $e_b(t) < e_a(t)$ holds, $e(t) < t$ holds as well, and since $e_a(0) = e_b(0)= 0$ holds, $e(0) = 0$ and $h(0)= d(e_a(0)) = d(0) = 1$ holds as well.
    Since the assumptions from proposition~\ref{thm:correction_both_general} are satisfied, we are done.
\end{proof}

We have established that \algcorrectone in the limit recovers the ground-truth signal and the degradation function. The following remark describes the rate of this convergence.

\begin{remark}[Rate of convergence]
Both \algcorrectone and \algcorrectboth, in case when $\expa(t) = t$ and $\expb(t) = \frac{t}{k}$ holds, converge in $\mathcal{O}(\log \frac{t}{\delta}) = \mathcal{O}(\log t)$ up to any fixed precision $\delta > 0$. This trivially follows from $\frac{t}{k^{n+1}} = \delta$.
\end{remark}

We examined how to extend our analysis from noiseless to noisy measurement model.
Since the noise is assumed to be additive and Gaussian, and the expectation of a reciprocal Gaussian random variable does not exist, we were not able to show above claims for the noisy measurement model.
However, empirical results suggest that it can be generalized to the noisy case.
We suspect this might be due to noise reduction in $\fitcurve$ procedure.
Furthermore, the assumption on $d(0) = 1$ was proven to be important by empirical analysis.

%% file: data_fusion.tex
Here we suppose that signals have been corrected for degradation and are thus given by two time series of measurements of the ground-truth signal $\sigs(t)$, $\timeseries{\siga_c}{k}{\signa}$ and $\timeseries{\sigb_c}{k}{\signb}$.
Measurements are obtained through a noisy measurement process, defined by \cref{eq:fusion_model_noise} (cf.~\cref{eq:degradation_model_noise}\footnote{In fact, since each measurement is corrected by the multiplication of the inverse value of the degradation, the variance of the initial noise changes with degradation, e.g. for signal $\siga$ to $\sigma_\siga(t) = \frac{\sigma_\siga}{d_c(\expa(t))^2}$. However, our data-fusion framework can easily be extended for such case when designing the $\diag(\vec{\sigma}^2(t))$.}), where $\noisea$ and $\noiseb$ are white noise signals and are independent of each other.
\begin{align}
\label{eq:fusion_model_noise}
\begin{aligned}
\siga_c(t) &= \sigs(t) + \noisea(t), && \noisea(t) \sim \mathcal{N}(0,\,\sigma_\siga^{2}), \\
\sigb_c(t) &= \sigs(t) + \noiseb(t), && \noiseb(t) \sim \mathcal{N}(0,\,\sigma_\sigb^{2}).
\end{aligned}
\end{align}
We would like to utilize the information embedded in the observations $\timeseries{\siga_c}{k}{\signa}$ and $\timeseries{\sigb_c}{k}{\signb}$ to obtain posterior belief about the signal $\sigs$.

The underlying process of $\sigs$ is assumed to be Brownian motion.
Kalman filter proposed by \citet{kalman1960} provides a powerful and efficient estimation method.
However, it is limited to discrete-time process model.
Since the process model of $s$ is random and not known, we would also like to incorporate time differences between two consecutive measurements.
Gaussian processes provide a powerful framework for this problem.
From now on we assume that $\sigs$ is a Gaussian process with zero mean and a covariance function $k(\cdot, \cdot)$, specified in \cref{eq:gp_definition}.
\begin{align}
\label{eq:gp_definition}
\begin{aligned}
s(t) \sim GP(0, k(t, t'))
\end{aligned}
\end{align}
The zero mean is assumed here since we either do not have or do not want to incorporate any prior knowledge about mean function $\mu$, thus it is common to consequently set it to $0$, i.e. $\mu(\vec x) \equiv \vec0$.
The covariance or the kernel function parameters, denoted by $\params$, are selected by maximizing log-marginal likelihood.

\subsection{Gaussian Processes}

Now, before proceeding, the reader should recall the fundamentals of Gaussian processes (GP), which are in detail described by \citet{rasmussen2005}.
In what follows we propose an advanced application of GPs.
Let us denote the set of observations of Gaussian Process $\sigs$ by $\mathcal{D} = \{(x_i, y_i)\}_{i=1}^n = (\vec{x}, \vec{y})$, where for each $i$ we know which sensor produced the observation.
We generalize the formulation of noise component $\sigma^2 I$ of the covariance matrix to incorporate observations from arbitrarily many different measurement sensor, each with its intrinsic variance of noise.
We replace it with $\diag(\bm{\sigma} ^2)$, where $\bm{\sigma} = [\sigma_1, \ldots, \sigma_n]^\T$ such that
\begin{align*}
    \sigma^2_i = 
    \begin{cases}
        \sigma_a^2 & \text{ if measurement $i$ came from sensor $\siga$,}\\
        \sigma_b^2 & \text{ if measurement $i$ came from sensor $\sigb$,}\\
        \sigma_c^2 & \text{ if measurement $i$ came from sensor $c$,}\\
        & \vdots\\
    \end{cases}
\end{align*}

Since the analysis when differentiating between measuring instruments is similar, but its notation is more tedious, we use simpler notation which assumes that all observations are produced by a single sensor. 

The main limitation of GPs is that given $n$ observations, we need to compute the inverse of a $n \times n$ matrix.
Time complexity of such operation is $\mathcal{O}(n^3)$, which is not scalable, especially when computational resources are limited.
Therefore, in the case when we have millions or even billions of data points a rather cumbersome downsampling has to be performed.
However, we would still like to use the idea of Gaussian processes, so we turn to sparse Gaussian processes (SGP), a much more scalable approach.

To tackle the problem with scalability, a lower bound for $\log p(\vec y|\vec x)$ can be constructed by approximating the exact Gaussian process with its sparse counterpart.
\citet{2016arXiv160604820B} prove the following:
\begin{align*}
    \log p(\vec y|\vec x) \ge -\frac{1}{2}\vec y^T(Q_{\params} + \sigma^2I)^{-1}\vec y - \frac{1}{2}\log |Q_{\params} + \sigma^2I| - \frac{n}{2}\log(2\pi) - \frac{1}{2\sigma^2}\tr(k_{\params}(\vec x, \vec x) - Q_{\params}),
\end{align*}
where vector $\vec u$ is a vector of $m \leq n$ inducing points (pseudo-observations) and $Q_{\params} = k_{\params}(\vec x, \vec u) \cdot k_{\params}(\vec u, \vec u)^{-1} \cdot k_{\params}(\vec u, \vec x)$ holds.
Moreover, \citet{2016arXiv160604820B} showed that right-hand side can be computed in $\mathcal{O}(nm^2)$, which for $m \ll n$ becomes much more tractable. 
The idea then is to maximize the lower-bound with respect to parameters $\params$ and $\vec u$ and make statements about posterior distribution of $f$ based on the optimal parameters.

\citet{2019arXiv190303571B} showed that if we use RBF kernel, $m$ needs to be of order $\log n$ in order for approximation error to go to 0, whereas in the case of Matern kernel with paramter $k + 1/2$, $m$ needs to be of order $n^{1/(2k+1)}$ that the approximation error tends towards 0.
Empirical observations showed that if the number of inducing points is small (in our case a couple of hundreds) then the model is not capable to capture fast fluctuations and instead finds some global trends.
With increasing number of inducing points the model can also capture fast fluctuations.
Moreover, by increasing $m$, Gaussian processes do not overfit the dataset (do not show any additional small scale behavior) if fast fluctuations do not exist.
Therefore, the data-fusion framework with SGPs enables us to observe the signal on multiple timescales, i.e. observing long-term and short-term signal properties.

%% file: results.tex
In practice both correction methods $\algcorrectone$ and $\algcorrectboth$ perform well and converge.
However, we observe that the former has a slightly faster convergence and thus we use it in all experiments.
In all experiments we use $\smoothmonmodel$ model, however \expmodel and \explinmodel yield similar results.

We first evaluate proposed methods on a synthetic dataset, where the ground-truth signal $s$ is known and is generated by simulating a Brownian motion of length $n$, and signals $a$ and $b$ are the sub-sampled versions of $s$ with added degradation effect.
Figure~\ref{fig:gen_noiseless_signals_a_b} shows the ground-truth signal $s$, degraded signals $a$ and $b$ (left), and signal obtained after running the iterative correction algorithm.
It clearly shows that corrected signals converge to the ground-truth signal.
\begin{figure}[htbp]
    \centering
    \subfloat[Raw signals]{{\includegraphics[width=0.45\textwidth]{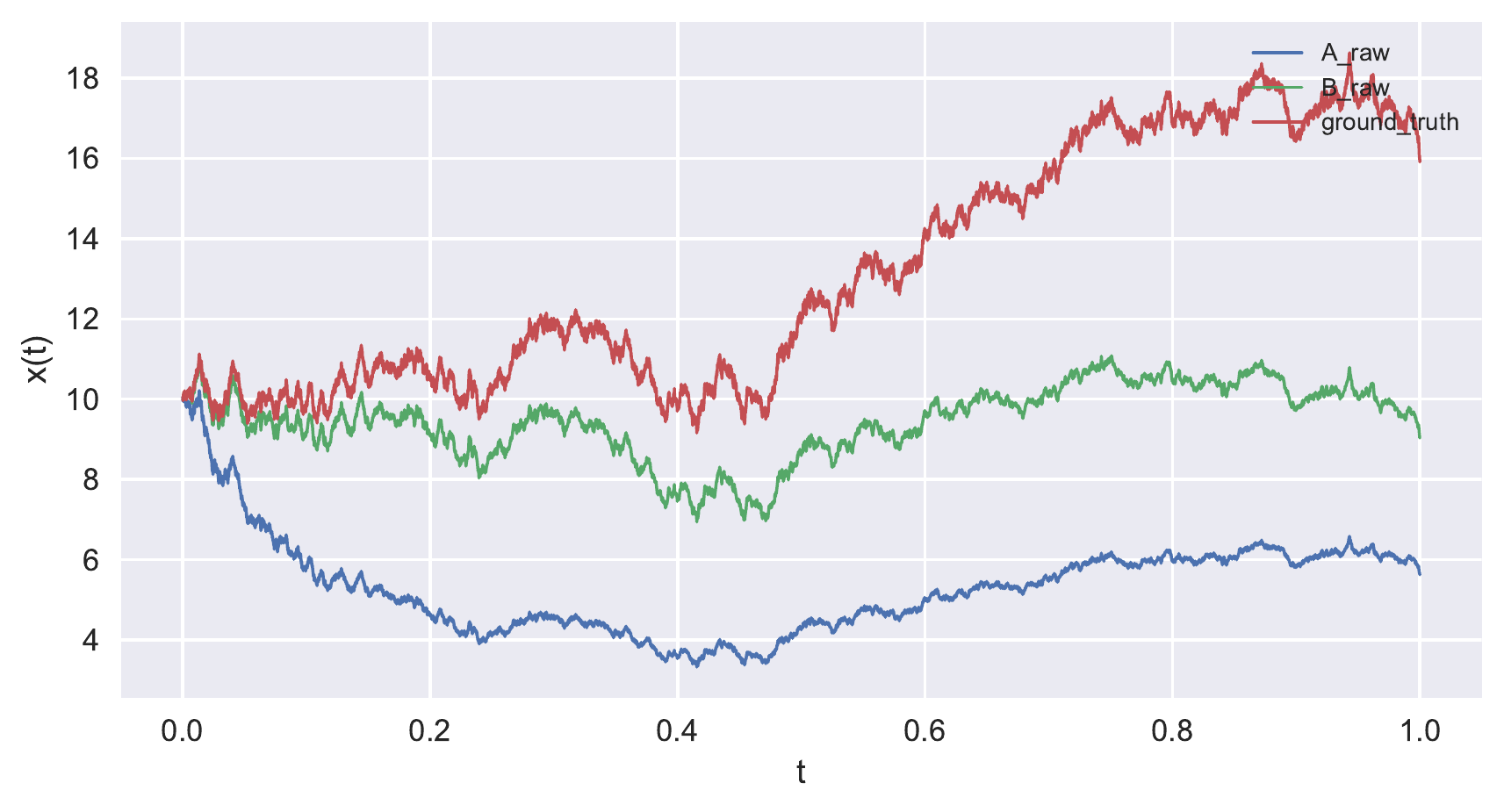} }}%
    \qquad
    \subfloat[Corrected signals]{{\includegraphics[width=0.45\textwidth]{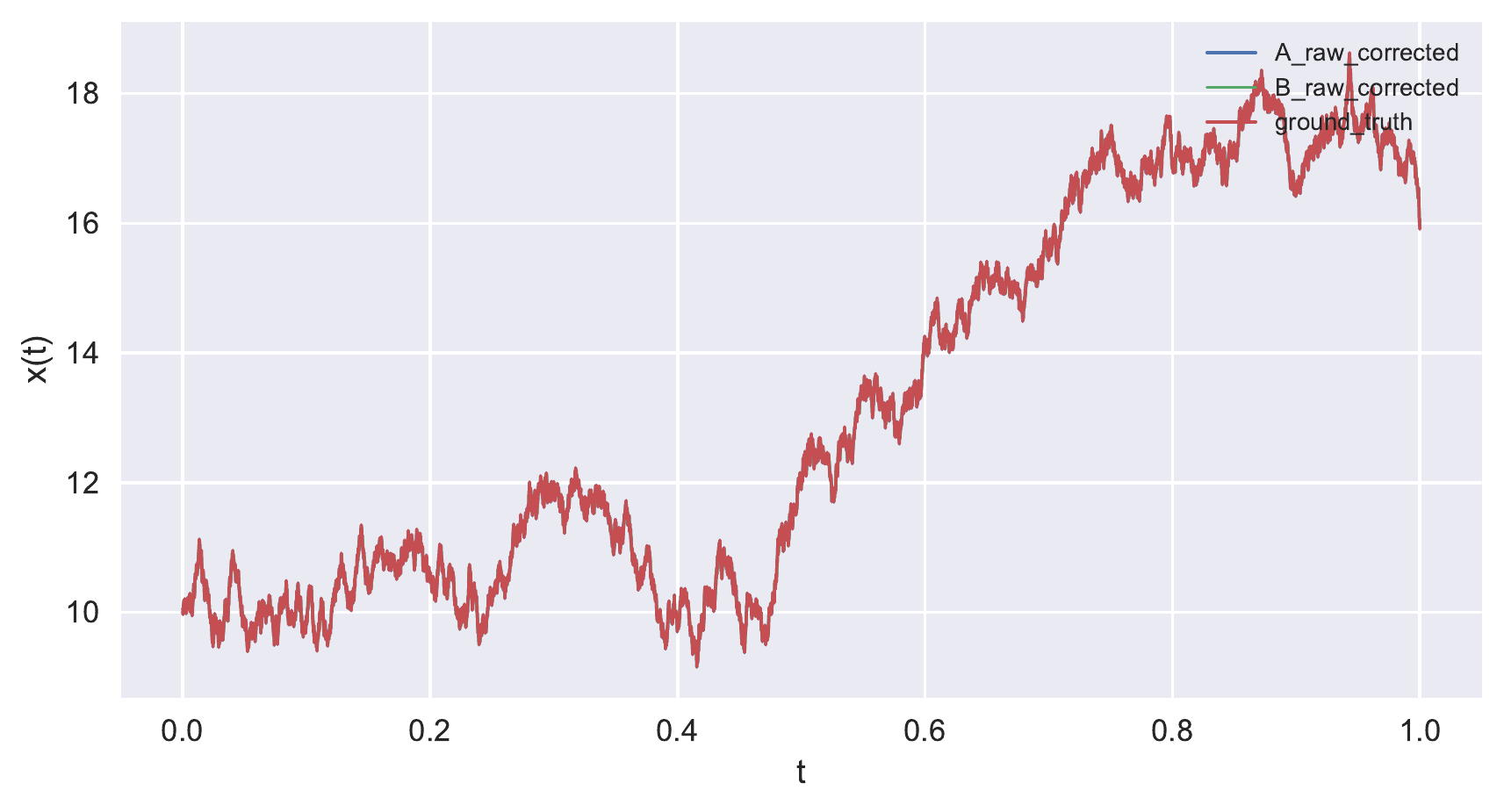}}}%
    \caption{Raw (left) and corrected (right) noise-free synthetic signals.}%
    \label{fig:gen_noiseless_signals_a_b}%
\end{figure}

Next, we consider the general case with additive white noise present.
These signals are shown on the left side of \cref{fig:gen_signals_a_b}, whereas on the right side the initial ratio between signals $a$ and $b$ is visualized.
The ratio is clearly not monotonically decreasing.
However, we argue that all information for degradation correction and retrieval of $s$ is embedded in it.
\begin{figure}[htbp]
    \centering
    \subfloat[Raw signals]{{\includegraphics[width=0.45\textwidth]{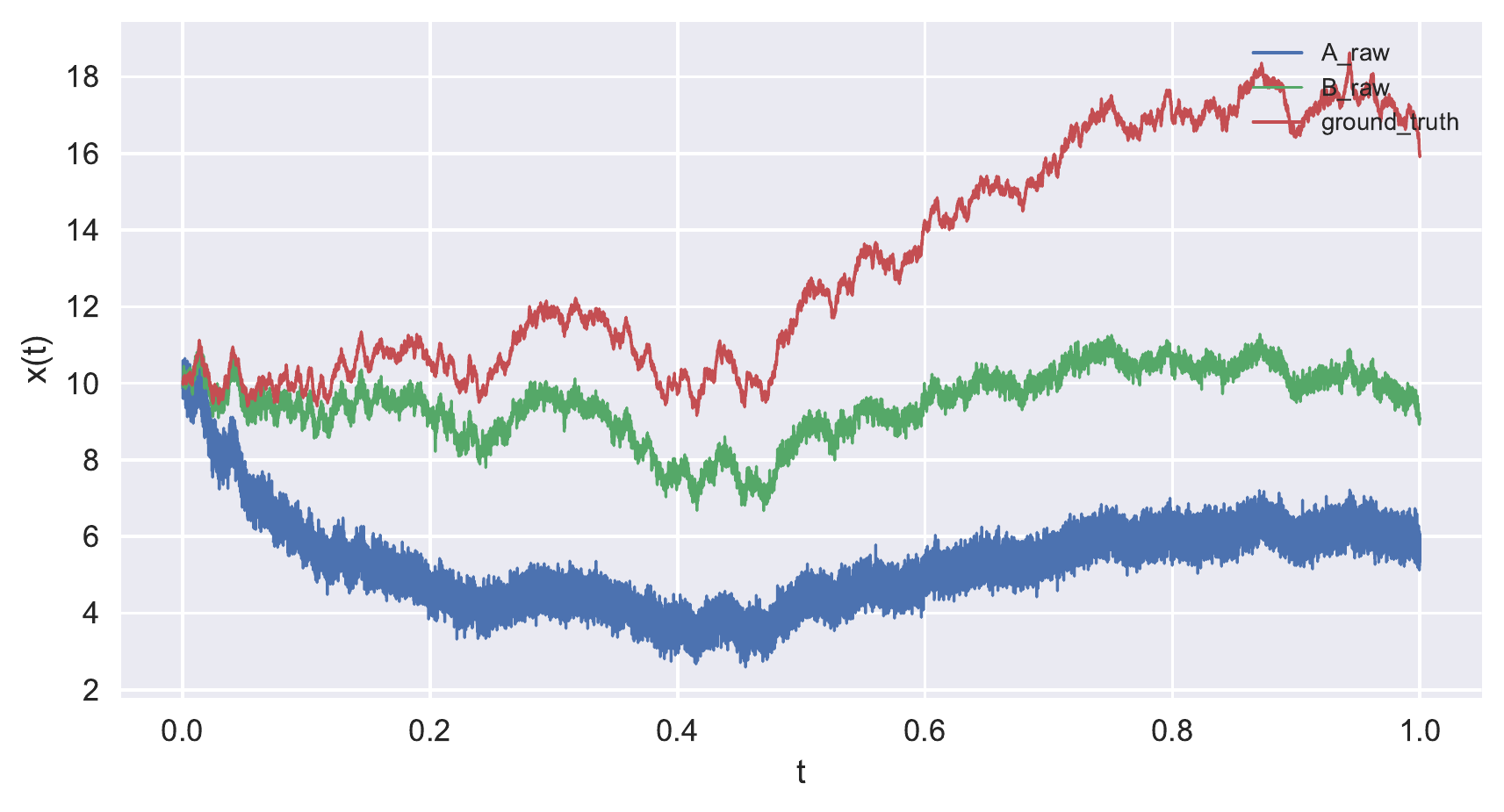} }}%
    \qquad
    \subfloat[Raw ratio]{{\includegraphics[width=0.45\textwidth]{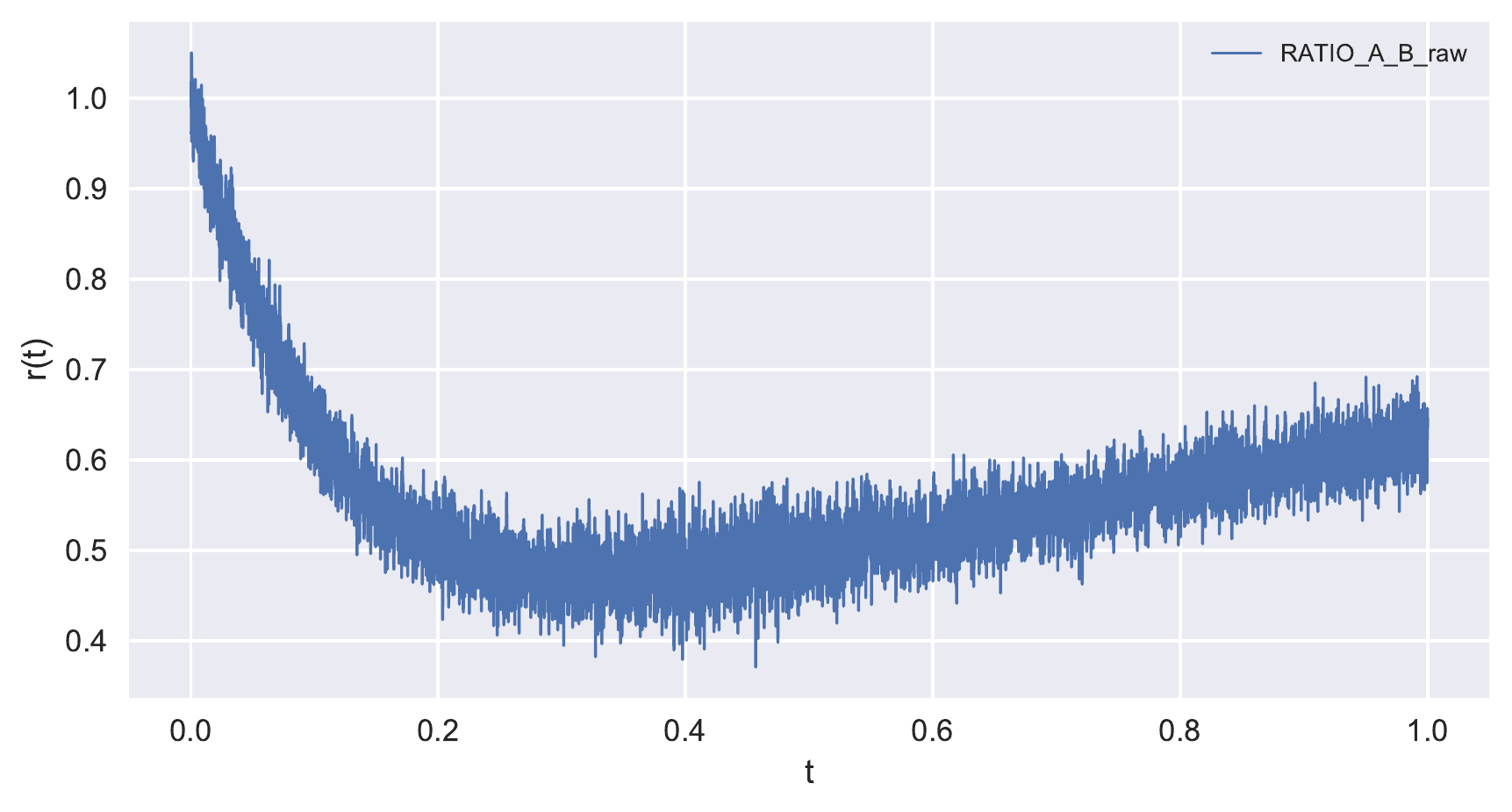}}}%
    \caption{Raw synthetic signals and their ratio together with the ground-truth signal.}%
    \label{fig:gen_signals_a_b}%
\end{figure}
\Cref{fig:gen_corrected_signals_a_b} shows that after performing corrections both signals effectively converge to the ground-truth signal, i.e. the means of noisy signals converge to $\sigs$.
Moreover, the ratio of corrected signals converges to a constant unit function.
The green line represents degradation of signal $a$, which is indeed a monotonically decreasing function.
\begin{figure}[htbp]
    \centering
    \subfloat[Corrected and ground-truth signals]{{\includegraphics[width=0.45\textwidth]{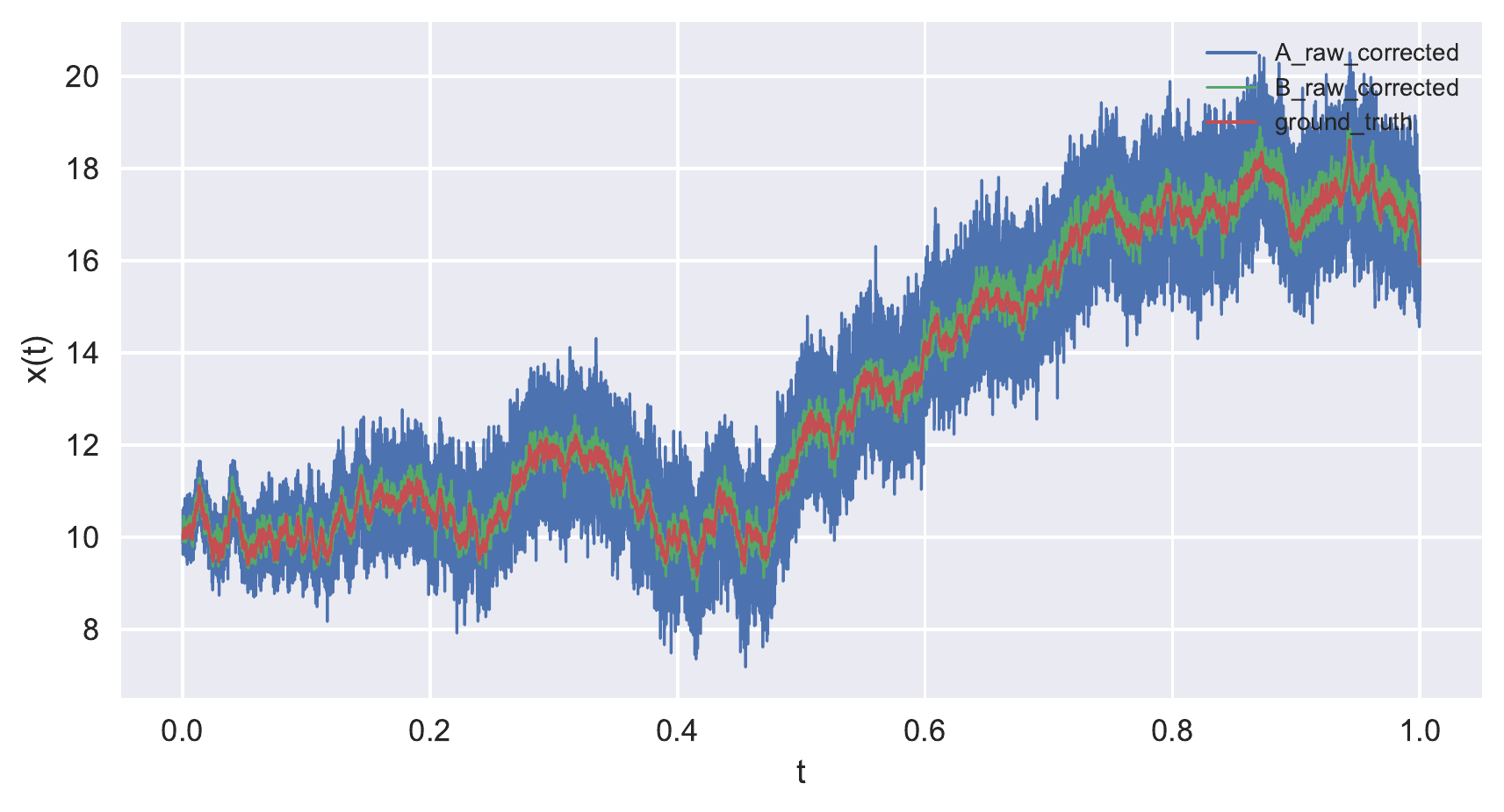}}}%
    \qquad
    \subfloat[Raw and corrected ratios]{{\includegraphics[width=0.45\textwidth]{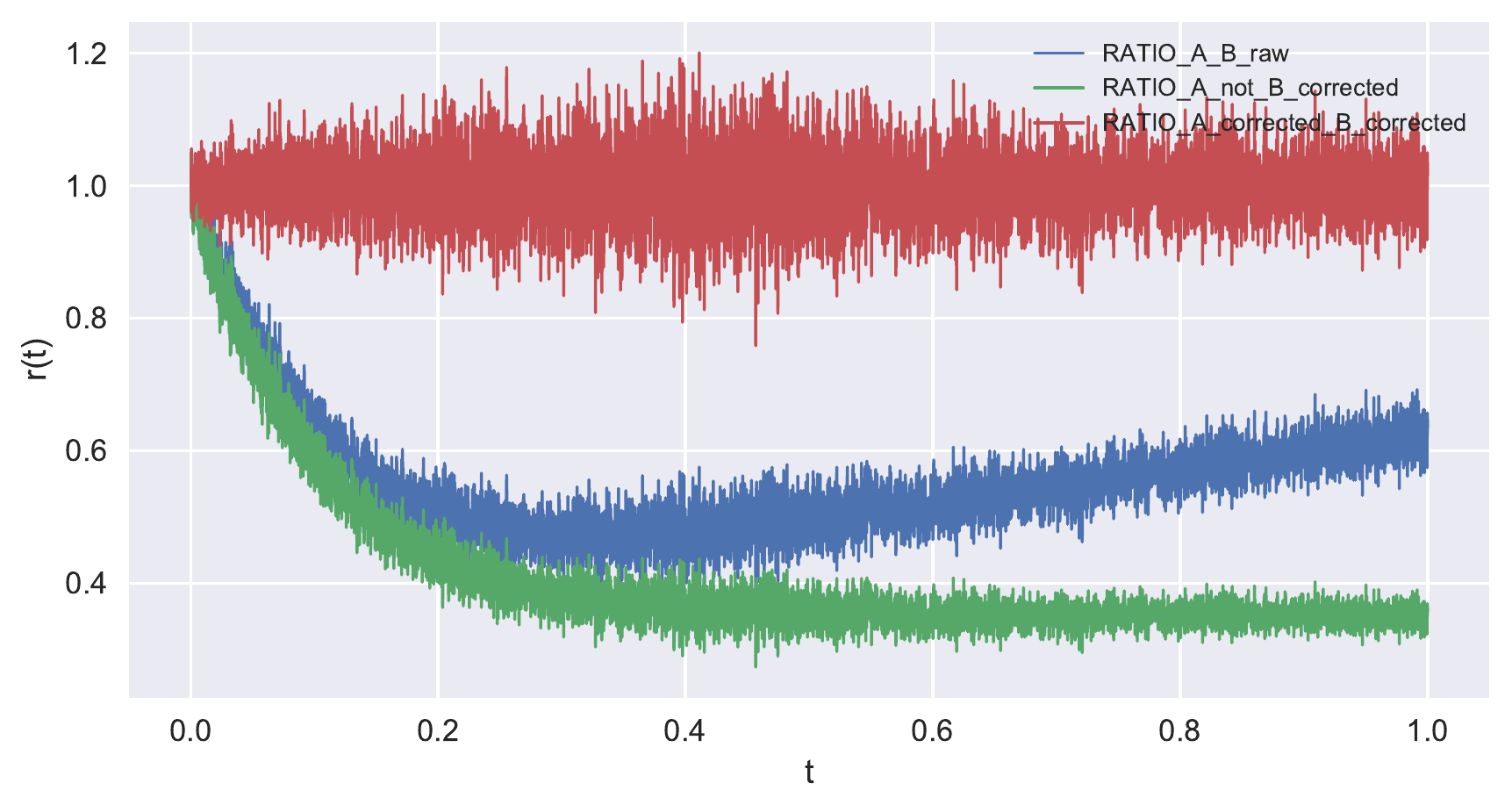}}}%
    \caption{Corrected synthetic signals, and comparison of raw and corrected ratios using $\smoothmonmodel$.}%
    \label{fig:gen_corrected_signals_a_b}%
\end{figure}

Although corrected signals are centered around the ground-truth signal, the correction has amplified the noise, with the amplification being greater as $t$ increases.
This effect was expected, since correction consists of divisions by a value smaller than 1.

We apply the proposed data-fusion method based on SGPs on the corrected signals and visualize the output signals.
The bold line represents the predicted mean and the light area the $95 \, \%$ confidence interval.
In our experiments we use the Matern kernel for the covariance function with $\nu = \frac{1}{2}$, $k(t, t') = \sigma^2 \cdot e^{-\frac{|t - t'|}{l}}$.
In \cref{fig:gen_output_signals_a_b} the output signal together with the data points obtained after correction is plotted.
The output signal fits the ground-truth very well, however, it is a bit smoother because the small number of inducing points cannot capture fast fluctuations.
\begin{figure}[htbp]
    \centering
    \subfloat[$m = 500$]{{\includegraphics[width=0.45\textwidth]{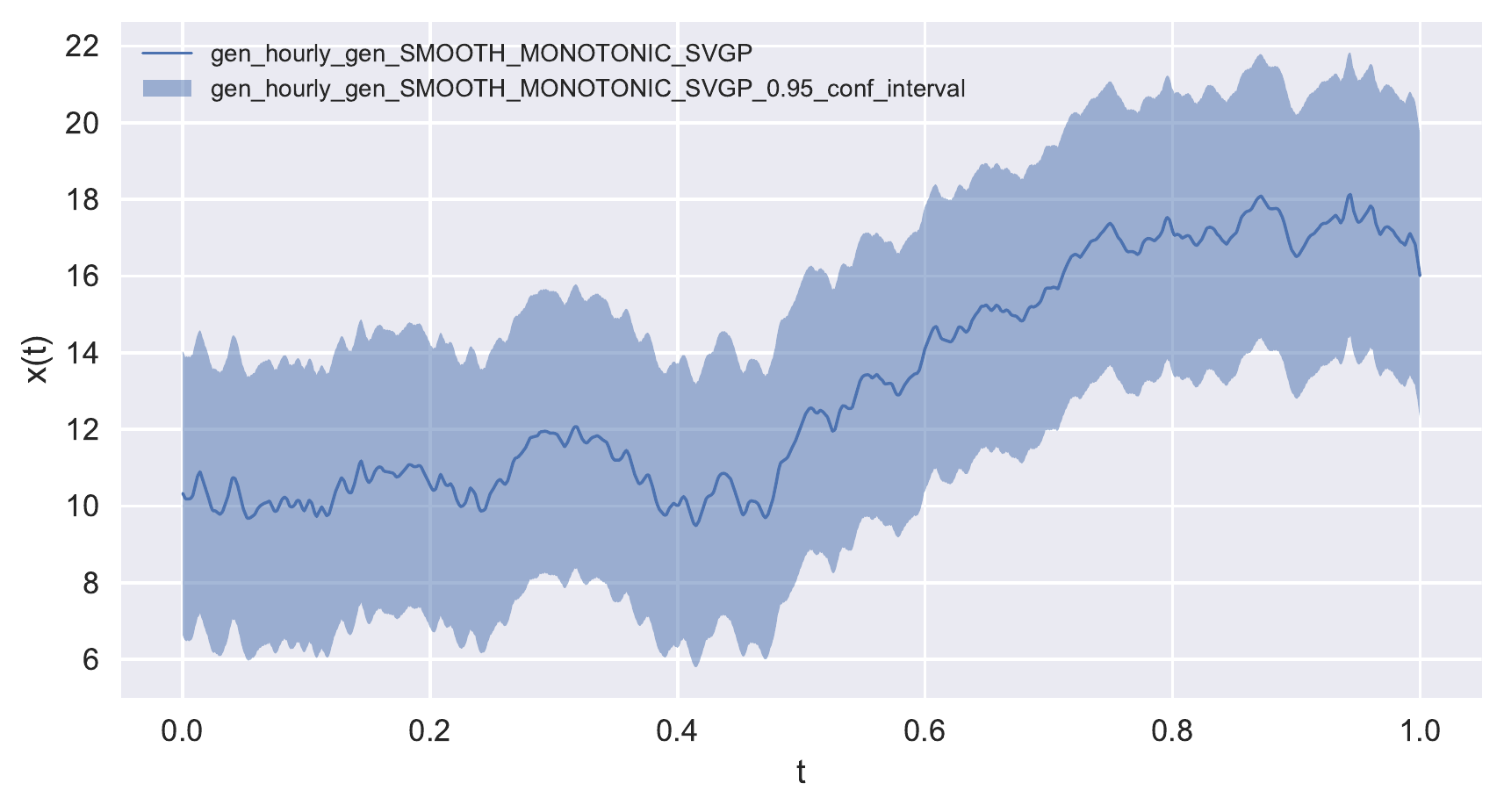}}}%
    \qquad
    \subfloat[Output and ground-truth signals and data points]{{\includegraphics[width=0.45\textwidth]{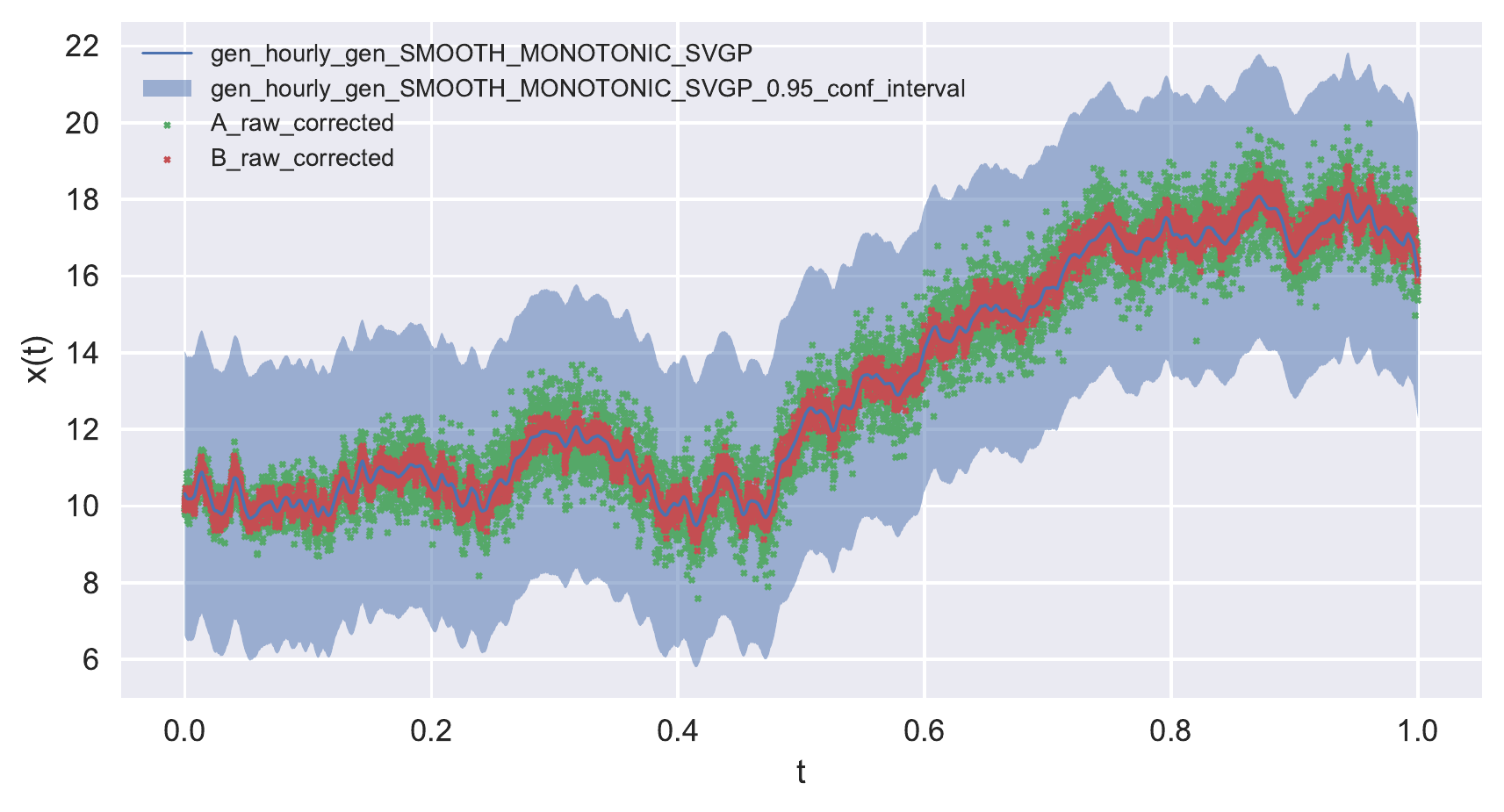}}}%
    \caption{Ground-truth estimation with $95\, \%$ confidence interval using $\svgp$ and $m = 500$.}%
    \label{fig:gen_output_signals_a_b}%
\end{figure}

Sparse Gaussian processes provide a simple and powerful framework for observing the signal on multiple timescales by varying the number of inducing points $m$, which can be seen in \cref{fig:gen_output_signals_a_b} and \cref{fig:gen__multiscale_output_signals_a_b} with $m \in \{100, 300, 500 \}$.
We can observe how by increasing $m$, the level of detail increases.

\begin{figure}[htbp]
    \centering
    \subfloat[$m = 100$]{{\includegraphics[width=0.45\textwidth]{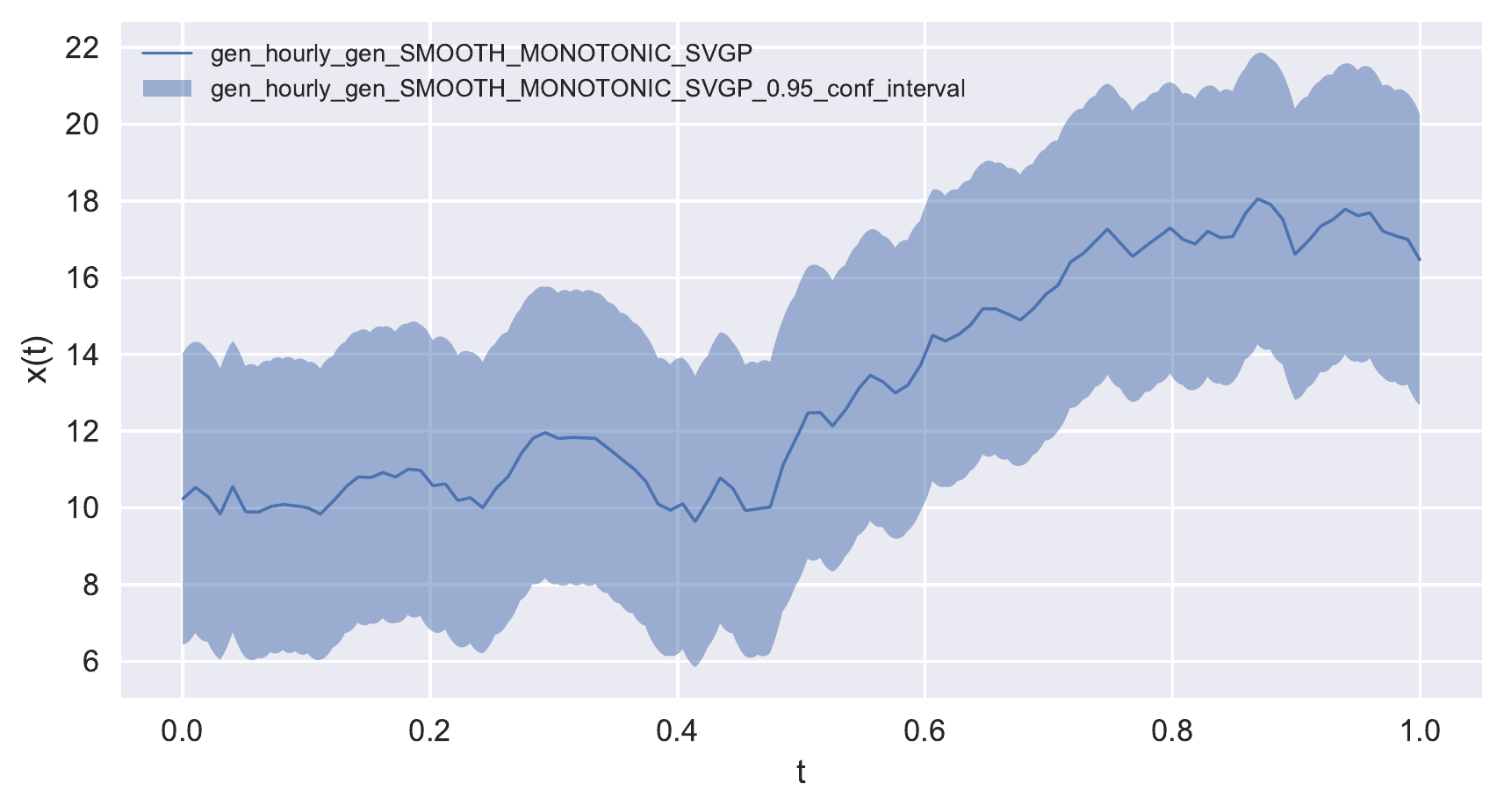}}}%
    \qquad
    \subfloat[$m = 300$]{{\includegraphics[width=0.45\textwidth]{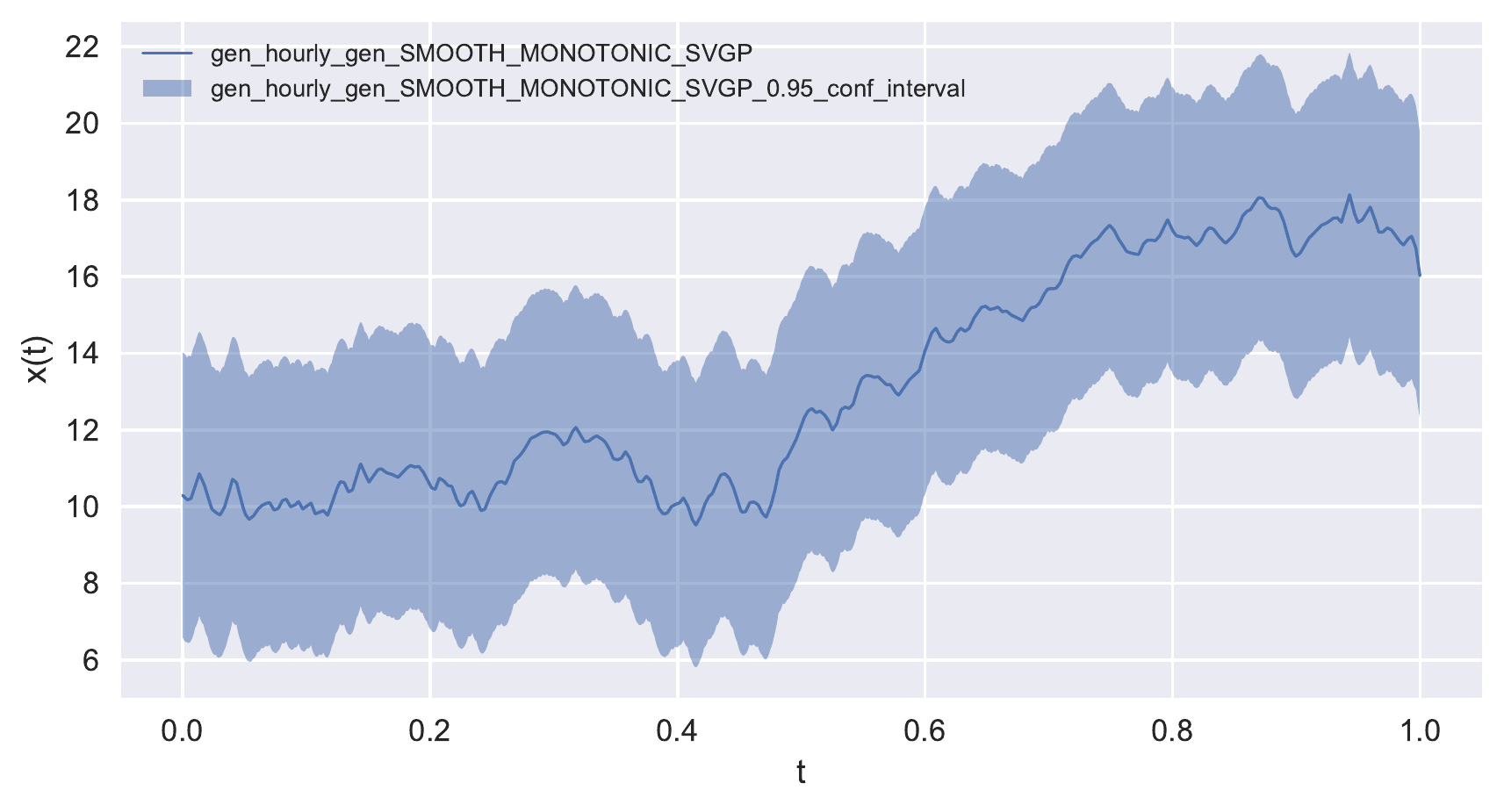}}}%
    \caption{Comparison of observing the ground-truth signal on multiple timescales by varying the number of inducing points $m$ for $m \in \{100, 300\}$ and $m = 500$ in \cref{fig:gen_output_signals_a_b}.}%
    \label{fig:gen__multiscale_output_signals_a_b}%
\end{figure}

Finally, in \cref{fig:gen_convergence} we visualize the first three steps of running the iterative correction algorithm.
If signal $b$ is estimated well, correction for that step is better as well.
This is the main intuition underlying both correction methods.
\begin{figure}[htbp]
    \centering
    \includegraphics[width=1.0\textwidth]{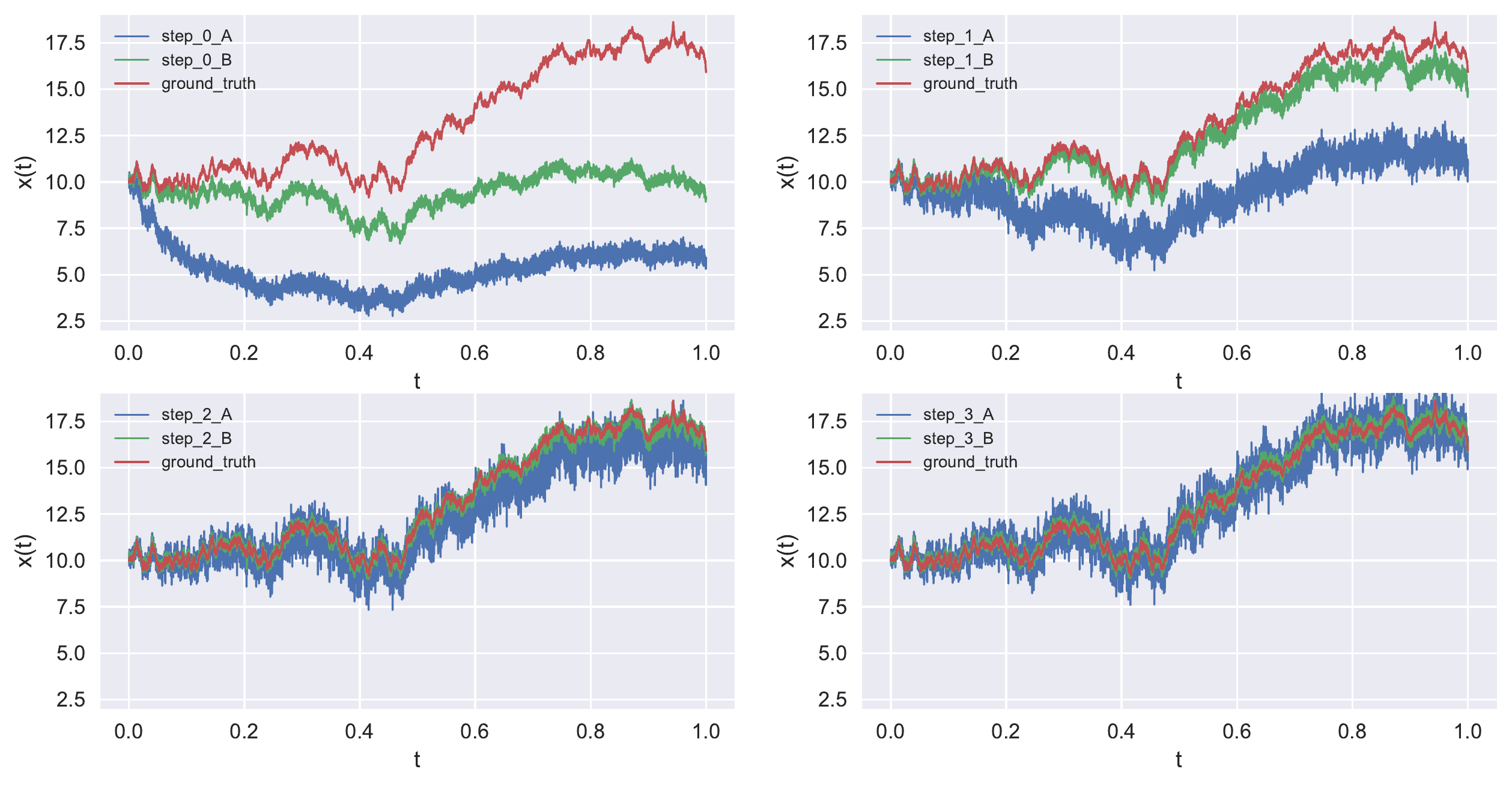}
    \caption{Convergence history.}
    \label{fig:gen_convergence}
\end{figure}

%% file: conclusion.tex
In the paper we presented a framework how to obtain a posterior belief about the measured signal, given measurements which came from at least two sensors which were exposed different amounts of time. 
The method consists of two parts; the first part serves to correct the degradation, while in the second part we fuse measurements from different sensors using Gaussian processes with a specialized kernel. 
We proved the convergence to the ground-truth for the first part of the method for the noiseless measurement model. 
We left for the future work to analyze the convergence in the presence of noise. 
To fuse the measurements we designed a special kernel which can incorporate the information that measurements come from sensors with possibly different measurement noise variances. 
Due to the large amount of data Sparse Gaussian processes were utilized instead of Gaussian processes, and empirically showed that by choosing large enough number of inducing points we obtain a decent approximation as well as observe the signal on multiple timescales.

\begin{ack}
    This paper is a result of Data Science Lab at ETH Zürich taking part in Autumn semester 2019.
    We would like to thank Prof. Dr. Andreas Krause and Prof. Dr. Ce Zhang for their supervision, invaluable insights, and suggestions on the modeling methods.
    Furthermore, we thank Dr. Wolfgang Finsterle for his domain knowledge support and data preparation.
\end{ack}

%% file: appendix.tex
\subsection{Degradation Iterative Correction}
\label{sec:ap_degradation_correction}

In \cref{section:DegradationModeling} we presented the $\algcorrectone$ iterative correction method. Here its counterpart is described. In $\algcorrectboth$ at every iteration the ratio between the corrected signals is computed and thus in the limit this ratio goes to $1$ as specified in \cref{eq:correct_both}.
\begin{align}
\label{eq:correct_both}
\tag{\algcorrectboth}
\begin{aligned}
r_n(\expa(t)) &= \frac{\siga_n(t)}{\sigb_n(t)} &
\siga_{n+1}(t) &= \frac{\siga_n(t)}{r_n(\expa(t))} &
\sigb_{n+1}(t) &= \frac{\sigb_n(t)}{r_n(\expb(t))} && \text{for $n = 0, 1, \ldots$}
\end{aligned}
\end{align}

\begin{algorithm}[ht]
\caption{$\algcorrectboth(\timeseries{\siga}{k}{\signm}, \timeseries{\sigb}{k}{\signm}, \timeseries{\expa}{k}{\signm}, \timeseries{\expb}{k}{\signm})$} \label{alg:correct_both}
\begin{algorithmic}[1]

\While{$\text{not converged}$} \Comment{E.g. $\norm{\siga_{i+1} - \siga_{i}}{2} /\norm{\siga_{i}}{2} + \norm{\sigb_{i+1} - \sigb_{i}}{2} / \norm{\sigb_{i}}{2} > \varepsilon$}
    \State $r \gets \frac{\siga}{\sigb}$ \Comment{Divide signals $\siga$ and $\sigb$ pointwise, i.e. $r[k] = \frac{\siga[k]}{\sigb[k]}$ $\forall \, k \in [\signm]$}
    \State $f(\cdot) \gets \fitcurve(\timeseries{\expa}{k}{\signm}, \timeseries{r}{k}{\signm})$  \Comment{Learn mapping $f \,:\, \expa \mapsto f(\expa)$}
    \State $\siga \gets \frac{\siga}{f(\expa)}$ \Comment{Correction update of signal $\siga$}
    \State $\sigb \gets \frac{\sigb}{f(\expb)}$ \Comment{Correction update of signal $\sigb$}
\EndWhile

\State $\siga_c \gets \siga; \; \sigb_c \gets \sigb$ \Comment{Corrected signals, $\siga(t) \approx \sigs(t)$ and $\sigb(t) \approx \sigs(t)$}
\State $r_c \gets \frac{\siga}{\sigb_c}$ \Comment{Divide signals $\siga$ and $\sigb_c$ pointwise with $r_c(\expa) \approx d(\expa)$}
\State $d_c(\cdot) \gets \fitcurve(\timeseries{\expa}{k}{\signm}, \timeseries{r_c}{k}{\signm})$ \Comment{Learn degradation function $d(\cdot)$}
\State \textbf{return} $\timeseries{\siga_c}{k}{\signm}, \timeseries{\sigb_c}{k}{\signm}, d_c(\cdot)$ \Comment{Return corrected signals and degradation function}
\end{algorithmic}
\end{algorithm}

\subsection{Convergence Theorems}
\label{sec:ap_convergence_theorems}

Proving the convergence of $\algcorrectboth$ goes along the similar lines as for $\algcorrectone$. For completeness, these proofs are stated next.

\begin{proposition}[$\algcorrectboth$ for $\expa(t) = t$ and $\expb(t) = \frac{t}{2}$]
    \label{firstTheorem}
    Let $a_0(t) = s(t) \cdot d(t)$ and $b_0(t) = s(t)  \cdot d(\frac{t}{2})$ for $t \geq 0$, where $s(t)>0$ is the ground-truth signal and $d \,:\, \R_{\geq 0} \to [0,1]$ is a continuous degradation function with $d(0) = 1$.
    If we run algorithm for $n = 0,1,\ldots$ ~:
    \begin{align*}
        r_n(t) = \frac{a_n(t)}{b_n(t)}, ~~
        a_{n+1}(t) = \frac{a_n(t)}{r_n(t)}, ~~
        b_{n+1}(t) = \frac{b_n(t)}{r_n(\frac{t}{2})},
    \end{align*}
    then it holds $\forall \, t \geq 0: \lim_{n \to \infty} a_n(t) = \lim_{n \to \infty}b_n(t) = s(t)$.
\end{proposition}
\begin{proof}
Let us fix an arbitrary $t > 0$.
Then we observe
\begin{align*}
    r_n(t) = \frac{a_n(t)}{b_n(t)} = \frac{a_{n-1}(t)}{r_{n-1}(t)}\frac{r_{n-1}(\frac{t}{2})}{b_{n-1}(t)} = r_{n-1}\left(\frac{t}{2}\right) = \cdots = r_0\left(\frac{t}{2^n}\right),
\end{align*}
where the last equality follows by induction.
Now, let us focus on the sequence $b_n(t)$:
\begin{align}
    b_n(t) &= \frac{b_{n-1}(t)}{r_{n-1}(\frac{t}{2})} = \frac{b_{n-1}(t)}{r_0(\frac{t}{2^{n}})} \nonumber\\
    &= \frac{b_{n-2}(t)}{r_0(\frac{t}{2^{n}}) \cdot r_{n-2}(\frac{t}{2})} = \frac{b_{n-2}(t)}{r_0(\frac{t}{2^{n}}) \cdot r_0(\frac{t}{2^{n-1}})} \nonumber \\
    &= \frac{b_0(t)}{\prod_{i=1}^{n}r_0(\frac{t}{2^{i}})}. \label{productFormula}
\end{align}
Calculation of ratio $r_0(t)$ gives us
\begin{align*}
    r_0(t) = \frac{a_0(t)}{b_0(t)} = \frac{s(t) \cdot d(t)}{s(t)  \cdot d(\frac{t}{2})} = \frac{d(t)}{d(\frac{t}{2})}.
\end{align*}
Inserting the latter observation and $b_0(t) = s(t) \cdot d(\frac{t}{2})$ into equation~\eqref{productFormula} yields
\begin{align*}
    b_n(t) = \frac{s(t) \cdot  d(\frac{t}{2})}{\prod_{i=1}^{n}\frac{d(\frac{t}{2^{i}})}{d(\frac{t}{2^{i+1}})}} = 
    s(t) \cdot d\left(\frac{t}{2}\right)\frac{\prod_{i=1}^{n}d(\frac{t}{2^{i+1}})}{\prod_{i=1}^{n}d(\frac{t}{2^{i}})} = 
    s(t) \cdot \frac{\prod_{i=1}^{n+1}d(\frac{t}{2^{i}})}{\prod_{i=1}^{n}d(\frac{t}{2^{i}})} = s(t) \cdot d\left(\frac{t}{2^{n+1}}\right).
\end{align*}
The proof is established by sending $n \xrightarrow{} \infty$, which gives
\begin{align*}
    \lim_{n \to \infty}b_n(t) = \lim_{n \to \infty} s(t)  \cdot d\left(\frac{t}{2^{n+1}}\right) = s(t)  \cdot  d\left(\frac{t}{\lim_{n \to \infty}2^{n+1}}\right) = s(t)  \cdot d(0) = s(t).
\end{align*}
\end{proof}
\begin{remark}
    The fact that $\lim_{n \to \infty} a_n(t) = \lim_{n \to \infty}b_n(t)$ holds is easily obtained from
    \begin{align*}
        a_{n+1}(t) = \frac{a_n(t)}{r_n(t)} = \frac{a_n(t) \cdot b_n(t)}{a_n(t)} = b_n(t).
    \end{align*}
\end{remark}

As for $\algcorrectone$ we have a direct corollary for an arbitrary $k > 1$.

\begin{corollary}[$\algcorrectboth$ for $\expa(t) = t$ and $\expb(t) = \frac{t}{k}$, $k > 1$]
    Let $a_0(t) = s(t) \cdot d(t)$ and $b_0(t) = s(t)  \cdot d(\frac{t}{k})$ for $t \geq 0$, where $s(t)>0$ is the ground-truth signal, $d \,:\, \R_{\geq 0} \to [0,1]$ is a continuous degradation function with $d(0) = 1$ and $k > 1$ is an arbitrary sampling rate parameter. If we run algorithm for $n = 0,1,\ldots$ ~:
    \begin{align*}
        r_n(t) = \frac{a_n(t)}{b_n(t)}, ~~
        a_{n+1}(t) = \frac{a_n(t)}{r_n(t)}, ~~
        b_{n+1}(t) = \frac{b_n(t)}{r_n(\frac{t}{k})},
    \end{align*}
    then it holds $\forall \, t \geq 0: \lim_{n \to \infty} a_n(t) = \lim_{n \to \infty}b_n(t) = s(t)$.
\end{corollary}
\begin{proposition}[$\algcorrectboth$ for $\expa(t) = t$ and $\expb(t) = e(t)$, $e(t) < t$]
    \label{secondProposition}
    Let $a_0(t) = s(t) \cdot d(t)$ and $b_0(t) = s(t)  \cdot d(e(t))$ for $t \geq 0$, where $s(t)>0$ is the ground-truth signal, $d \,:\, \R_{\geq 0} \to [0,1]$ is a continuous degradation function with $d(0) = 1$ and $e(t)$ is the exposure function of signal $b$, for which it holds $e(0) = 0$ and $e(t) < t$ for all $t > 0$. If we run algorithm for $n=0,1,\ldots$ ~:
    \begin{align*}
        r_n(t) = \frac{a_n(t)}{b_n(t)}, ~~
        a_{n+1}(t) = \frac{a_n(t)}{r_n(t)}, ~~
        b_{n+1}(t) = \frac{b_n(t)}{r_n(e(t))},
    \end{align*}
    then it holds $\forall \, t \geq 0: \lim_{n \to \infty} a_n(t) = \lim_{n \to \infty}b_n(t) = s(t)$.
\end{proposition}
\begin{proof}
    Let us fix an arbitrary $t > 0$ and compute $b_1(t)$, which gives
    \begin{align*}
        b_1(t) = \frac{b_0(t)}{r_0(e(t))} = s(t) \cdot d(e(t)) \cdot \frac{d(e(e(t)))}{d(e(t))} = s(t) \cdot d(e(e(t))).
    \end{align*}
    Let $e^n(t) = \underbrace{(e \circ e\circ \cdots \circ e)}_{n \text{ times}}(t)$, then by induction we have
    \begin{align*}
        b_n(t) = s(t) \cdot d(e^{n+1}(t)).
    \end{align*}
    By taking the limit $n \xrightarrow{} \infty$, we get
    \begin{align*}
        \lim_{n \to \infty}b_n(t) &= \lim_{n \to \infty}s(t) \cdot d\left(e^{n+1}(t)\right) = s(t) \cdot d\left(\lim_{n \to \infty}e^{n+1}(t)\right).
    \end{align*}
    Let $\eta = \lim_{n \to \infty}e^{n+1}(t)$, then we have
    \begin{align*}
        \eta = \lim_{n \to \infty}e^{n+1}(t) = e(\lim_{n \to \infty}e^{n}(t)) = e(\eta)
    \end{align*}
    but since $e(\eta) < \eta$ for $\eta > 0$ and $e(0) = 0$ we obtain $\eta = 0$. Hence we conclude this proof as
    \begin{align*}
         \lim_{n \to \infty}b_n(t) = s(t) \cdot d(0) = s(t).
    \end{align*}
\end{proof}
\begin{theorem}
    \label{thm:correctboth}
    Let $a_0(t) = s(t) \cdot d(e_a(t))$ and $b_0(t) = s(t)  \cdot d(e_b(t))$ for $t \geq 0$, where $s(t)>0$ is the ground-truth signal, $d \,:\, \R_{\geq 0} \to [0,1]$ is a continuous degradation function with $d(0) = 1$ and $e_a(t), e_b(t):[0, \infty)\to[0, \infty)$ are the continuous exposure function of signal $a$ and $b$ respectively.
    Let us further assume $e_a(0) = e_b(0)= 0$, $e_b(t) < e_a(t)$ for all $t > 0$ and that there exist function $e_a^{-1}:[0, \infty)\to[0, \infty)$.
    If we run algorithm for $n=0,1,\ldots$ ~:
     \begin{align*}
        r_n(t) = \frac{a_n(t)}{b_n(t)}, ~~
        a_{n+1}(t) = \frac{a_n(t)}{r_n(t)}, ~~
        b_{n+1}(t) = \frac{b_n(t)}{r_n((e_a^{-1}\circ e_b)(t))},
    \end{align*}
    then it holds $\forall \, t \geq 0: \lim_{n \to \infty} a_n(t) = \lim_{n \to \infty}b_n(t) = s(t)$.
\end{theorem}
\begin{proof}
    Let $h(t) = d(e_a(t))$. Then it holds $d(e_b(t)) = h(e_a^{-1}\circ e_b)(t)$. If we denote $e = e_a^{-1}\circ e_b$, then the proposed algorithm transforms to:
    \begin{align*}
        r_n(t) = \frac{a_n(t)}{b_n(t)}, ~~
        a_{n+1}(t) = \frac{a_n(t)}{r_n(t)}, ~~
        b_{n+1}(t) = \frac{b_n(t)}{r_n(e(t))},
    \end{align*}
    with the initial setting: $a_0(t) = s(t) \cdot h(t)$ and $b_0(t) = s(t)  \cdot h(e(t))$.
    Since $e_b(t) < e_a(t)$ holds, $e(t) < t$ holds as well, and since $e_a(0) = e_b(0)= 0$ holds, $e(0) = 0$ and $h(0)= d(e_a(0)) = d(0) = 1$ holds as well.
    Since the assumptions from proposition~\ref{secondProposition} are satisfied, we are done.
\end{proof}